\documentclass[twoside]{article}

\usepackage[preprint]{aistats2022}
\usepackage[round]{natbib}

\usepackage{url}

\usepackage{blindtext}

\usepackage{amsmath, amssymb}
\usepackage{mathtools}
\usepackage{bm}

\newcommand{\Rbb}{\mathbb{R}}

\newcommand{\Tbb}{\mathbb{T}}

\newcommand{\Ncal}{\mathcal{N}}

\newcommand{\Xcal}{\mathcal{X}}

\newcommand{\Ical}{\mathcal{I}}

\newcommand{\eg}{e.g.~}
\newcommand{\ie}{i.e.~}

\newcommand{\pderiv}[2]{\frac{\partial #1}{\partial #2}}

\newcommand{\diff}{\,\text{d}}
\DeclareMathOperator{\diag}{diag}

\DeclareMathOperator{\gp}{GP}

\DeclareMathOperator*{\argmax}{arg\,max}

\DeclareMathOperator{\blockdiagonal}{blockdiag}

\usepackage{xcolor}

\usepackage{amsthm}

\usepackage[capitalise]{cleveref}
\crefname{appendix}{Supplement}{Supplements}

\newtheorem{proposition}{Proposition}

\newtheorem{assumption}[proposition]{Assumption}

\theoremstyle{remark}

\newtheorem{example}[proposition]{Example}

\usepackage{xr}

\usepackage{microtype}      

\usepackage{algorithm}
\usepackage{algpseudocode}

\usepackage{graphicx}

\begin{document}

%

%
\runningauthor{Nicholas Krämer, Nathanael Bosch, Jonathan Schmidt, Philipp Hennig}

\twocolumn[

    \aistatstitle{Probabilistic ODE Solutions in Millions of Dimensions}

    \aistatsauthor{
        Nicholas Krämer\textsuperscript{*,1}
        \And
        Nathanael Bosch\textsuperscript{*,1}
        \And
        Jonathan Schmidt\textsuperscript{*,1}
        \And
        Philipp Hennig\textsuperscript{1,2}
    }

    \aistatsaddress{
        \textsuperscript{1}University of Tübingen\\
        \textsuperscript{2}Max Planck Institute for Intelligent Systems, Tübingen, Germany
    }
]

\begin{abstract}
    Probabilistic solvers for ordinary differential equations (ODEs) have emerged as an efficient framework for uncertainty quantification and inference on dynamical systems.
    In this work, we explain the mathematical assumptions and detailed implementation schemes behind solving {high-dimensional} ODEs with a probabilistic numerical algorithm.
    This has not been possible before due to matrix-matrix operations in each solver step, but is crucial for scientifically relevant problems---most importantly, the solution of discretised {partial} differential equations.
    In a nutshell, efficient high-dimensional probabilistic ODE solutions build either on independence assumptions or on Kronecker structure in the prior model.
    We evaluate the resulting efficiency on a range of problems, including the probabilistic numerical simulation of a differential equation with millions of dimensions.
\end{abstract}

\section{INTRODUCTION}

\begin{figure*}[ht]
    \begin{center}
        \includegraphics{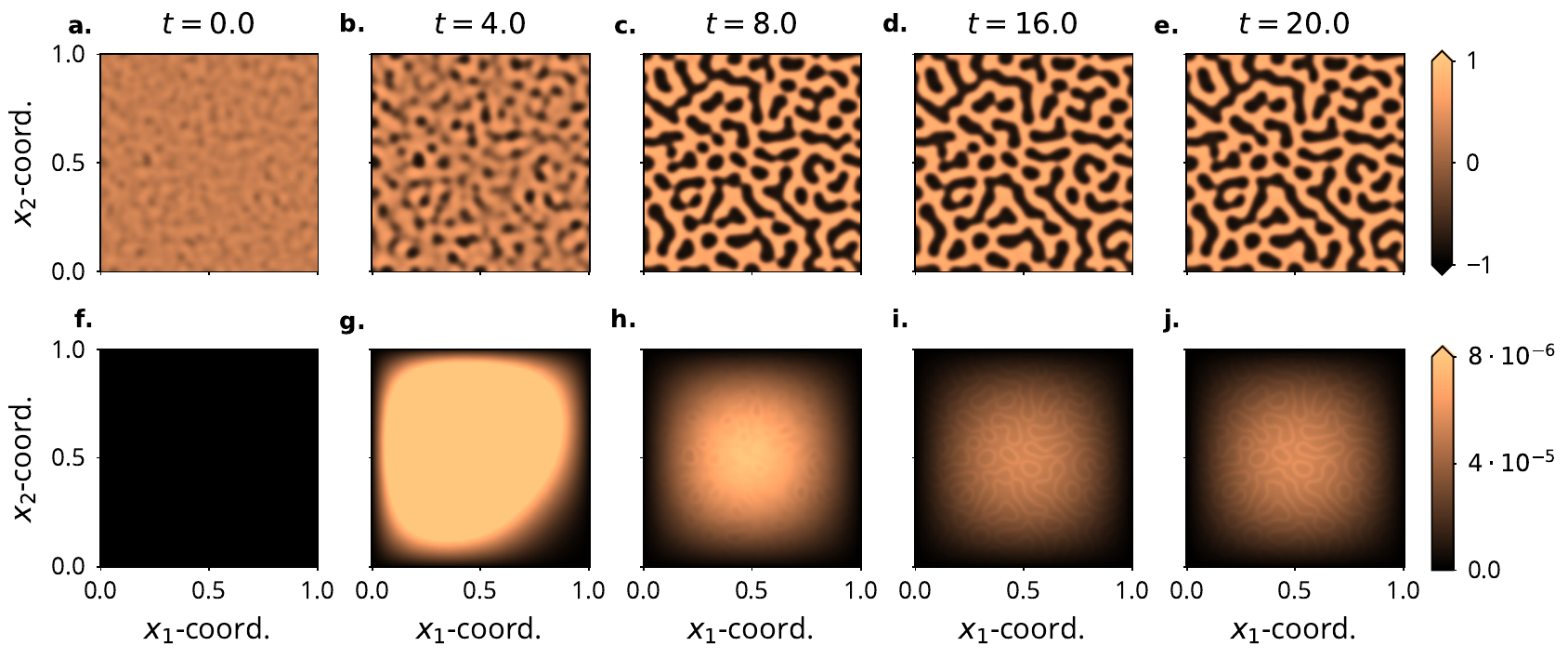}
    \end{center}
    \caption{\textit{Simulating a high-dimensional ODE:} Probabilistic solution of a discretised FitzHugh-Nagumo PDE model \citep{ambrosio2009propagation}. Means (a-e) and standard-deviations (f-j), $t_0=0$ (left) to $t_\text{max}=20$ (right). The patterns in the uncertainties match those in the solution. The simulated ODE is 125k-dimensional.}
    \label{fig:figure_1_pde}
\end{figure*}

\paragraph{Problem Statement}
\label{sec:introduction}
This paper discusses a class of algorithms that computes the solution of initial value problems based on ordinary differential equations (ODEs), \ie finding a function $y$ that satisfies
\begin{align}\label{eq:ivp}
    \dot y(t) = f(y(t), t),
\end{align}
for all $ t \in [t_0, t_\text{max}]$,
as well as the initial condition $y(t_0) = y_0 \in \Rbb^d$.
Usually, $f$ is non-linear, in which case the solution of \cref{eq:ivp} cannot generally be derived in closed form and has to be approximated numerically.
We continue the work of \emph{probabilistic} numerical algorithms for ODEs \citep{schober2019probabilistic,tronarp2019probabilistic,kersting2020convergence,tronarp2021bayesian,bosch2021calibrated,kramer2020stable}.
Like other filtering-based ODE solvers (``ODE filters''), the algorithm used herein translates numerical approximation of ODE solutions to a problem of probabilistic inference.
The resulting (approximate) posterior distribution quantifies the uncertainty associated with the unavoidable discretisation error \citep{bosch2021calibrated} and provides a language that integrates well with other data inference schemes \citep{kersting2020differentiable,schmidt2021probabilistic}.
The main difference to prior work is that we focus on the setting where the dimension $d$ of the ODE is high, that is, say, $d \gg 100$. (It is not clearly defined at which point an ODE counts as high-dimensional, but $d \approx 100$ is already a scale of problems in which previous state-of-the-art probabilistic ODE solvers faced computational challenges.)

\paragraph{Motivation And Impact}
High-dimensional ODEs describe the interaction of large networks of dynamical systems and appear in many disciplines in the natural sciences.
The perhaps most prominent example arises in the simulation of discretised partial differential equations.
There, the dimension of the ODE equals the number of grid points used to discretise the problem \citep[with \eg finite differences;][]{schiesser2012numerical}.
More recently, ODEs gained popularity in machine learning through the advent of neural ODEs \citep{chen2018neural}, continuous normalising flows \citep{grathwohl2018ffjord}, or physics-informed neural networks \citep{raissi2019physics}.
With the growing complexity of the model, each of the above can quickly become high-dimensional.
If such use cases shall gain from probabilistic solvers, fast algorithms for large ODE systems are crucial.

\paragraph{Prior Work And State-of-the-Art}
Many non-probabilistic ODE solvers, for example, explicit Runge--Kutta methods, have a computational complexity linear in the ODE dimension $d$ \citep{hairer1993solving}.
Explicit Runge--Kutta methods are often the default choices in ODE solver software packages.
Compared to the efficiency of the methods provided by DifferentialEquations.jl \citep{rackauckas2017differentialequations}, SciPy \citep{virtanen2020scipy}, or Matlab \citep{shampine1997matlab}, probabilistic methods have lacked behind so far.
Intuitively, ODE filters are a fusion of ODE solvers and Gaussian process models---two classes of algorithms that suffer from high dimensionality.
More precisely, the problem is that probabilistic solvers require matrix-matrix operations at each step.
The matrices have $O(d^2)$ entries, which leads to $O(d^3)$ complexity for a single solver step and has made the solution of high-dimensional ODEs impossible. ODE filters are essentially nonlinear, approximate Gaussian process inference schemes (with a lot of structure). As in the GP community \cite[e.g.][]{JMLR:v6:quinonero-candela05a}, the path to low computational cost in these models is via factorisation assumptions.

\paragraph{Contributions}
Our main contribution is to prove in which settings ODE filters admit an implementation in $O(d)$ complexity.
Thereby, they become a class of algorithms comparable to explicit Runge--Kutta methods not only in estimation performance \citep[error contraction as a function of evaluations of $f$;][]{kersting2020convergence,tronarp2021bayesian} but also in computational complexity (cost per evaluation of $f$).
The resulting algorithms deliver uncertainty quantification and other benefits of probabilistic ODE solvers on high-dimensional ODEs
(see \cref{fig:figure_1_pde}. The ODE from this figure will be explored in more detail in \cref{sec:experiments}).
The key novelties of the present work are threefold:
\begin{enumerate}
    \item
          \textit{Acceleration via independence:}
          A-priori, ODE filters commonly assume independent ODE dimensions \citep[e.g.][]{kersting2020convergence}.
          We single out those inference schemes that naturally preserve independence.
          Identification of independence-preserving ODE solvers is helpful because each ODE dimension can be updated separately.
          The performance implications are that a single matrix-matrix operation with $O(d^2)$ entries is replaced with $d$ matrix-matrix operations with $O(1)$ entries.
          In other words, $O(d)$ instead of $O(d^3)$ complexity for a single solver step.
          This is \cref{prop:complexity_diagonal_jacobian}.

    \item
          \textit{Calibration of multivariate output-scales:}
          A single ODE system often models the interaction between states that occur on different scales. It is useful to acknowledge differing output scales in the ``diffusivity'' of the prior (details below). We generalise the calibration result by \citet{bosch2021calibrated} to the class of solvers that preserve the independence of the dimensions.
          This is \cref{prop:vector-valued-time-varying-diffusion}.

    \item
          \textit{Acceleration via Kronecker structure:}
          Sometimes, prior independence assumptions may be too restrictive.
          For instance, one might have prior knowledge of correlations between ODE dimensions (\cref{ex:spatiotemporal-model} in \cref{sec:kronecker_structure_preserved}).
          Fortunately, a subset of probabilistic ODE solvers can exploit \emph{and preserve} Kronecker structure in the system matrices of the state space. Preserving the Kronecker structure brings over the performance gains from above to dependent priors. This is \cref{prop:complexity-kronecker}.
\end{enumerate}
Additional minor contributions are detailed where they occur.
To demonstrate the scalability of the resulting algorithm, the experiments in \cref{sec:experiments} showcase simulations of ODEs with dimension $d\sim 10^7$.

\section{ODE FILTER SETUP}
\label{sec:setup}

The following section details the technical setup of an ODE filter, including the prior (\cref{subsec:prior}) and information model (\cref{subsec:information_model}), as well as a selection of relevant practical considerations (\cref{subsec:practical_considerations}).

\subsection{Prior Model}
\label{subsec:prior}

The following is standard for probabilistic ODE solvers, and therefore essentially identical to the presentation by \eg \citet{schober2019probabilistic}.
Herein, however, we place a stronger emphasis on Kronecker- and independence-structures in the system matrices compared to prior work.
Both are important for the theoretical statements below.
\citet{sarkka2013bayesian} or \citet*{sarkka2019applied} provide a comprehensive explanation of the mathematical concepts regarding inference in state-space models.

\paragraph{Stochastic Process Prior On The ODE Solution}
Let \(Y := (Y^i)_{i=1}^d = \left(Y_0^i, \dots, Y_\nu^i \right)_{i=1}^d\) solve the linear, time-invariant stochastic differential equation (SDE)
\begin{align}\label{eq:prior_sde}
    \diff Y(t) = A Y(t) \diff t + B \diff W(t),
\end{align}
subject to a Gaussian initial condition
\begin{align}\label{eq:prior_initial_condition}
    Y(t_0) \sim \Ncal(m_0, C_0)
\end{align}
for some $m_0$ and $C_0 := \Gamma \otimes \breve{C}_0$.
The SDE is driven by a $d$-dimensional Wiener process $W$ with diffusion $\Gamma \in \Rbb^{d\times d}$ and governed by the system matrices
\begin{align}\label{eq:system_matrices}
    A := I_d \otimes \breve{A}, \quad \breve{A}:=\sum_{q=0}^{\nu-1} e_q e_{q+1}^\top, \quad
    B := I_d \otimes e_\nu,
\end{align}
where $e_q \in \Rbb^{\nu+1}$ is the $q$-th basis vector.
The zeroth component of $Y$, $(Y_0^i)_{i=1}^d$, is an integrated Wiener process.
With such $A$ and $B$, the $q$-th component $(Y_q^i)_{i=1}^d$ models the $q$-th derivative of the integrated Wiener process.
Similar SDEs can be written down for \eg the integrated Ornstein-Uhlenbeck process or the Mat\'ern process (the only differences would be additional non-zero entries in $A$).
If $\Gamma$ were diagonal, the Kronecker structure in $A$ and $B$ would imply prior pairwise independence between $Y^i$ and $Y^j$, $i \neq j$.
\cref{sec:independence_assumption} uses the diagonality assumption to reveal the efficient implementation of a class of ODE filters.
\cref{sec:kronecker_structure_preserved} allows $\Gamma$ to be any symmetric, positive definite matrix, which is why we do not make strong assumptions on $\Gamma$ yet.

\paragraph{Discretisation}

Let $\Tbb=(t_0, ..., t_N)$ be some time-grid with step-size $h_n := t_{n+1} - t_n$.
While for the presentation, we assume a fixed grid, practical implementations choose $t_n$ adaptively.
Reduced to $\Tbb$, due to the Markov property, the process $Y$ becomes
\begin{align} \label{eq:equivalent_discretisation}
    Y(t_{n+1}) \mid Y(t_n) \sim \Ncal(\Phi(h_n) Y_n, \Sigma(h_n))
\end{align}
for matrices $\Phi(h_n)$ and $\Sigma(h_n)$, which are defined as
\begin{subequations}\label{eq:phi_and_sigma}
    \begin{align}
        \Phi(h_n)   & = \exp(A h_n),                                                                  \\
        \Sigma(h_n) & = \int_0^{h_n} \Phi(h_n - \tau) B \Gamma B^\top\Phi(h_n - \tau)^\top\diff \tau.
    \end{align}
\end{subequations}
The definition of $\Phi(h_n)$ uses the matrix exponential.
$\Phi(h_n)$ inherits the block diagonal structure from $A$,
\begin{align}
    \Phi := I_d \otimes \breve{\Phi}(h_n), \quad \breve{\Phi}(h_n) = \exp(\breve{A} h_n),
\end{align}
and $\Sigma$ has a Kronecker factorisation similar to $C_0$,
\begin{subequations}
    \begin{align}
        \Sigma(h_n)
         & := \Gamma \otimes \breve{\Sigma}(h_n) ,                                                             \\
        \breve{\Sigma}(h_n)
         & := \int_0^{h_n} \breve{\Phi}(h_n - \tau) e_\nu e_\nu^\top \breve{\Phi}(h_n - \tau)^\top \diff \tau.
    \end{align}
\end{subequations}
The discretisation allows efficient extrapolation from $t_n$ to $t_{n+1}$.
Let $Y(t_n) \sim \Ncal(m_n, C_n)$.
Then,
\begin{align}\label{eq:predicted_rv}
    Y(t_{n+1}) \sim \Ncal(m_{n+1}^-, C_{n+1}^-)
\end{align}
with mean and covariance
\begin{subequations}\label{eq:predicted_mean_and_cov}
    \begin{align}
        m^-_{n+1} & = \Phi(h_n) m_n,                         \label{eq:extrapolate_mean}            \\
        C^-_{n+1} & = \Phi(h_n) C_n \Phi(h_n)^\top + \Sigma(h_n). \label{eq:extrapolate_covariance}
    \end{align}
\end{subequations}
For improved numerical stability, probabilistic ODE solvers compute this prediction in square root form, which means that only square root matrices of $C_n$ and $C_{n+1}^-$ are propagated without ever forming full covariance  matrices \citep{kramer2020stable,grewal2014kalman}.
\cref{sec:sqrt_implementation} recalls details about square root implementations of ODE filters.

\subsection{Information Model}
\label{subsec:information_model}

\paragraph{Information Operator}
The information operator
\begin{align}
    \Ical(Y)(t) := \dot Y(t) - f(Y(t), t),
\end{align}
known as the local defect \citep{gustafsson1992control},
captures ``how well (a sample from) $Y$ solves the given ODE''---if this value is large, the current state is an inaccurate approximation, and if it is small, $Y$ provides a good estimate of the truth.
The goal is to make the defect as small as possible over the entire time domain.

\paragraph{Artificial Data}
The local defect $\Ical$ can be kept small by conditioning $Y$ on $\Ical(Y)(t) \overset{!}{=}0$ on ``many'' grid-points. Due to the regular prior and the regularity-preserving information operator $\Ical$, conditioning the prior on a zero-defect leads to an accurate ODE solution \citep{tronarp2021bayesian}.
Altogether, the probabilistic ODE solver targets the posterior distribution
\begin{align} \label{eq:probabilistic_ode_solution}
    p\left(Y \,\left|\, \left\{\Ical(Y)(t_n) = 0\right\}_{n=0}^N, Y_0(t_0) = y_0\right.\right).
\end{align}
(Recall from \cref{eq:prior_sde} that lower indices in $Y$ refer to the derivative, \ie $Y_0$ is the integrated Wiener process, and $Y_q$ its $q$-th derivative.)
We call the posterior in \cref{eq:probabilistic_ode_solution} the \emph{probabilistic ODE solution}.
Unfortunately, a nonlinear vector field $f$ implies a nonlinear information operator $\Ical$.
Thus, the exact posterior is intractable.

\paragraph{Linearisation}
A tractable approximation of the probabilistic ODE solution is available through linearisation.
Linearising $f$ indirectly linearises $\Ical$, and the corresponding probabilistic ODE solution arises via Gaussian inference.
Let $F_y$ be (an approximation of) the Jacobian of $f$ with respect to $y$.
One can approximate the ODE vector field with a Taylor series
\begin{align}\label{eq:linearised_f}
    f(y) \approx \hat f_\xi(y) := f(\xi) + F_y(\xi)(y - \xi)
\end{align}
at some $\xi \in \Rbb^d$.
Let $E_q := I_d \otimes e_q$ be the projection matrix that extracts the $q$-th derivative from the full state $Y$.
In other words, $\dot Y_0 = E_1 Y$.
\cref{eq:linearised_f} implies a linearisation of $\Ical$ at some $\eta \in \Rbb^{d(\nu + 1)}$,
\begin{align}
    \Ical(Y)(t) \approx \hat \Ical_\eta(Y)(t) & := H(t) Y(t) + b(t),
\end{align}
with linearisation matrices
\begin{subequations}\label{eq:linearisation_matrices}
    \begin{align}
        H(t) & := E_1 - F_y(E_0 \eta, t) E_0,                 \\
        b(t) & := F_y(E_0 \eta, t) E_0 \eta - f(E_0 \eta, t).
    \end{align}
\end{subequations}
$\hat\Ical_\eta$ is linear in $Y(t)$.
Therefore, the approximate probabilistic ODE solution becomes tractable with Gaussian filtering and smoothing once $\hat\Ical_\eta$ is plugged into \cref{eq:probabilistic_ode_solution} \citep{sarkka2013bayesian,tronarp2019probabilistic}.
At $t_n$, $\eta$ is usually the predicted mean $m_n^-$, which yields the extended Kalman filter \citep{sarkka2013bayesian}.
For ODE filters, there are three relevant versions of $F_y$:
\begin{enumerate}
    \item \textit{EK0:} Use the zero-matrix to approximate the Jacobian, $F_y \equiv 0$, which has been a common choice since early work on ODE filters \citep{schober2019probabilistic,kersting2020convergence}, and implies a zeroth-order approximation of $f$ \citep{tronarp2019probabilistic}.
    \item \textit{EK1:} Use the full Jacobian $F_y = \nabla_y f$, which amounts to a first-order Taylor approximation of the ODE vector field \citep{tronarp2019probabilistic}. In its general form, the EK1 does not fit the assumptions made below and can thus does not immediately scale to high dimensions. Instead, we introduce the following variant:
    \item \textit{Diagonal EK1:}
          Use the diagonal of the full Jacobian, $F_y = \diag(\nabla_y f)$.
          This choice conserves the efficiency of the EK0 to a solver that uses Jacobian information.
          The diagonal EK1 is another minor contribution of the present work.
          The EK1 is more stable than the EK0 \citep{tronarp2019probabilistic}.
          \cref{subsec:stability-diagonal-ek1} empirically investigates how much stability using only the diagonal of the Jacobian provides.
\end{enumerate}

\paragraph{Measurement And Correction}
A probabilistic ODE solver step consists of an extrapolation, measurement, and correction phase.
Extrapolation has been explained in \cref{eq:predicted_rv,eq:predicted_mean_and_cov} above.
Denote
\begin{align}
    Y_{n+1}^- := Y(t_{n+1}) \sim \Ncal(m_{n+1}^-, C_{n+1}^-).
\end{align}
The measurement phase approximates
\begin{align}\label{eq:measured_rv}
    Z_{n+1} := \hat\Ical_{m_{n+1}^-}(Y_{n+1}^-)(t_{n+1}) \approx \Ical(Y_{n+1}^-)(t_{n+1})
\end{align}
by exploiting the linearisation matrices $H$ and $b$,
\begin{subequations}\label{eq:measured_rv_matrices}
    \begin{align}
        Z_{n+1} & \,\sim \Ncal(z_{n+1}, S_{n+1}),          \\
        z_{n+1} & := H(t_{n+1}) m_{n+1}^- + b(t_{n+1}),    \\
        S_{n+1} & := H(t_{n+1}) C_{n+1}^- H(t_{n+1})^\top.
    \end{align}
\end{subequations}
$Z_{n+1}$ will be used for calibration (details below).
The extrapolated random variable is then corrected as
\begin{subequations}\label{eq:correct_mean_and_cov}
    \begin{align}
        Y_{n+1} & \,\sim \Ncal(m_{n+1}, C_{n+1}),                                                               \\
        m_{n+1} & := m_{n+1}^- - C_{n+1}^- H(t_{n+1})^\top S_{n+1}^{-1} z_{n+1},     \label{eq:mean_correction} \\
        C_{n+1} & := \Xi \,C_{n+1}^-\,\Xi ^\top,              \label{eq:joseph_update1}                         \\
        \Xi     & := I -  C_{n+1}^- H(t_{n+1})^\top S_{n+1}^{-1}H(t_{n+1}).\label{eq:joseph_update2}
    \end{align}
\end{subequations}
The update in \cref{eq:joseph_update1,eq:joseph_update2} is the Joseph update \citep{bar2004estimation}. In practice, we never form the full $C_{n+1}$ but compute only the square root matrix by applying $\Xi$ to the square root matrix of $C_{n+1}^-$.
It is not a Cholesky factor (because it is not lower triangular), but generic square root matrices suffice for numerically stable implementation of probabilistic ODE solvers \citep{kramer2020stable}.

\subsection{Practical Considerations}
\label{subsec:practical_considerations}
Let us conclude with brief pointers to further practical considerations that are important for efficient probabilistic ODE solutions.
\begin{itemize}
    \item
          \textit{Initialisation:}
          The ODE filter state models a stack of a state and the first $\nu$ derivatives. The stability of the probabilistic ODE solver depends on the accurate initialisation of all derivatives. The current state of the art is to use Taylor-mode automatic differentiation \citep{kramer2020stable,griewank2008evaluating}, whose complexity scales exponentially with the dimension of the ODE.
          Instead, we initialise the solver by inferring
          \begin{align}
              p(Y(t_0) \mid Y_0(\tau_m) = \hat y(\tau_m), m=0, ..., \nu)
          \end{align}
          on $\nu+1$ small steps $\tau_0, ..., \tau_{\nu}$ where the $\hat y(\tau_m)$ are computed with \eg a Runge--Kutta method. This is a slight generalisation of the strategy used by \citet{schober2019probabilistic} (also refer to \citet{schober2014probabilistic,gear1980runge}), in the sense that we formulate this initialisation as probabilistic inference instead of setting the first few means manually.
    \item \textit{Error estimation:} Comprehensive explanation of error estimation and step-size adaptation is out of scope for the present work; we refer the reader to \citet{schober2019probabilistic} and \citet{bosch2021calibrated}.
\end{itemize}

\section{INDEPENDENT PRIOR MODELS ACCELERATE ODE SOLVERS}
\label{sec:independence_assumption}

This section establishes the main idea of the present work: \emph{probabilistic ODE solvers are fast and efficient when the prior models each dimension independently.}

\subsection{Assumptions}
Independent dimensions stem from a diagonal $\Gamma$.

\begin{assumption}\label{ass:diagonal_gamma}
    Assume that the diffusion $\Gamma$ of the Wiener process in \cref{eq:prior_sde} is a diagonal matrix.
\end{assumption}

\cref{ass:diagonal_gamma} implies that the initial covariance $C_0$ (\cref{eq:prior_initial_condition}) is the Kronecker product of a diagonal matrix with another matrix, thus block diagonal.
\cref{ass:diagonal_gamma} is not very restrictive; in prior work on ODE filters, $\Gamma$ was always either  $\Gamma = \gamma^2 I$ for some $\gamma > 0$ \citep{schober2019probabilistic,tronarp2019probabilistic,kersting2020convergence,bosch2021calibrated,tronarp2021bayesian,kramer2020stable}, or diagonal \citep{bosch2021calibrated}.

\subsection{Calibration}
Tuning the diffusion $\Gamma$ is crucial to obtain accurate posterior uncertainties.
As announced in \cref{sec:setup}, the mathematical assumptions for calibrating $\Gamma$ coincide with the assumptions that lead to an efficient ODE filter.
Thus, we discuss $\Gamma$ before proving the linear complexity of probabilistic solvers under \cref{ass:diagonal_gamma}.

\paragraph{Four Approaches}
Recall the observed random variable $Z_n$ (\cref{eq:measured_rv}).
ODE filters calibrate $\Gamma$ with quasi-maximum-likelihood-estimation (quasi-MLE): Consider the prediction error decomposition \citep{schweppe1965evaluation},
\begin{subequations}
    \begin{align}
        p(\{Z_n\}_{n=0}^N)
         & = p(Z_0) \prod_{n=0}^{N-1} p(Z_{n+1} \mid Z_n)                                                \\
         & \approx \Ncal(z_0, S_0) \prod_{n=0}^{N-1} \Ncal(z_{n+1}, S_{n+1}) \label{eq:approx_evidence}.
    \end{align}
\end{subequations}
$\Gamma$ is a quasi-MLE if it maximises \cref{eq:approx_evidence}.
The specific choice of calibration depends on respective model for $\Gamma$, and reduces to one of four approaches: on the one hand, fixing and calibrating a \emph{time-constant} $\Gamma$ versus allowing a \emph{time-varying} $\Gamma$; on the other hand, choosing a \emph{scalar} diffusion $\Gamma=\gamma^2 I$ versus choosing a \emph{vector-valued} diffusion $\Gamma=\diag(\gamma^1, ..., \gamma^d)$.
Roughly speaking, a time-varying, vector-valued diffusion allows for the greatest flexibility in the probabilistic model.
One contribution of the present work is to extend the vector-valued diffusion results by \citet{bosch2021calibrated} to a slightly broader class of solvers (\cref{prop:vector-valued-time-varying-diffusion} below).
Scalar diffusion will reappear in \cref{sec:kronecker_structure_preserved} below. Time-constant diffusion is addressed in \cref{sec:independence-time-constant-diffusion}.

\paragraph{Time-Varying Diffusion}

Allowing $\Gamma$ to change over the time-steps, all measurements before time $t_n$ are independent of $\Gamma_n$.
Under the assumption of an error-free previous state (which is common for hyperparameter calibration in ODE solvers), a local quasi-MLE for $\Gamma_n = \gamma_n^2 I$ arises as \citep{schober2019probabilistic}
\begin{align}\label{eq:time-varying-scalar-diffusion}
    \hat \gamma_n^2 := \frac{1}{d} z_n \left[H(t_n) \Sigma(h_n) H(t_n)^\top\right]^{-1} z_n.
\end{align}
This can be extended to a quasi-MLE for the EK0 with vector-valued $\Gamma_n = \diag(\gamma_n^1, ..., \gamma_n^d)$ \citep{bosch2021calibrated}
\begin{align}\label{eq:time-varying-vector-valued-diffusion}
    (\hat\gamma_n^i)^2 := (z_n^i)^2 / [H(t_n) \Sigma(h_n) H(t_n)^\top]_{ii},
\end{align}
for all $i=1,...,d$.
In this work, we generalise the EK0 requirement to \cref{ass:diagonal_gamma} and a diagonal Jacobian.

\begin{proposition}\label{prop:vector-valued-time-varying-diffusion}
    Under \cref{ass:diagonal_gamma} and for diagonal $F_y$, the estimators $(\hat\gamma_n^i)_i$ in \cref{eq:time-varying-vector-valued-diffusion} are quasi-MLEs.
\end{proposition}

\begin{proof}[Sketch of the proof.]
    Two ideas are relevant: (i) a diagonal Jacobian implies a block diagonal $H(t_n)$ and a diagonal $H(t_n) \Sigma(h_n) H(t_n)^\top$ (which will be proved formally in \cref{prop:complexity_diagonal_jacobian} below); (ii) the local evidence, \ie the probability of $\Ncal(z_n, S_n)$ being zero, decomposes into a sum over the coordinates.
    Maximising each summand with respect to $\gamma_n^i$ yields the claim.
\end{proof}
A very similar case can be made for time-constant diffusion (see  \cref{sec:independence-time-constant-diffusion}).
\cref{prop:vector-valued-time-varying-diffusion} is a generalisation of the results by \citet{bosch2021calibrated} in the sense that \cref{prop:vector-valued-time-varying-diffusion} is not restricted to the EK0.


\subsection{Complexity}
Now, with calibration in place, we can discuss the computational complexity of ODE filters under \cref{ass:diagonal_gamma}.
The following proposition establishes that
for diagonal Jacobians, a single solver step costs $O(d)$.

\begin{proposition} \label{prop:complexity_diagonal_jacobian}
    Suppose that \cref{ass:diagonal_gamma} is in place.
    If the Jacobian of the ODE is (approximated as) a diagonal matrix, then a single step with a filtering-based probabilistic ODE solver costs $O(d \nu^3)$ in floating-point operations, and $O(d \nu^2)$ in memory.
\end{proposition}

\begin{proof}
    Let $Y_n \sim \Ncal(m_n, C_n)$. Assume that $C_n$ is block diagonal. We show that block diagonality is preserved through a step, and since by \cref{ass:diagonal_gamma}, $C_0$ is block diagonal, we do not lose generality.
    Recall $\Phi(h_n)$ and $\Sigma(h_n)$ from \cref{eq:equivalent_discretisation,eq:phi_and_sigma}.
    $\Phi(h_n)$ is block diagonal, and since $\Gamma_n$ is diagonal, $\Sigma(h_n)$ is block diagonal.

    \emph{(i) Extrapolate the mean:}
    The mean is extrapolated according to \cref{eq:extrapolate_mean}, which costs $O(d \nu^2)$, because of the block diagonal $\Phi(h_n)$. Each dimension is extrapolated independently.

    \emph{(ii) Evaluate the ODE:}
    Next, $H=H(t_{n+1})$ and $b=b(t_{n+1})$ from \cref{eq:linearisation_matrices} are assembled, which involves evaluating $f$ and $F_y$ at $\xi := E_0 m_{n+1}^-$ ($E_0$ is a projection matrix and can be implemented as array indexing, so $\xi$ comes at negligibily low cost).
    $F_y = \diag(F_y^1, ... F_y^d)$ is a diagonal matrix, therefore
    \begin{align}
        H = \blockdiagonal(H^1, ..., H^d)
    \end{align}
    is block diagonal with blocks
    \begin{align}
        H^i := e_1 - e_0 F_y^i, \quad i=1,...,d
    \end{align}
    (recall the basis vectors $e_q$ from \cref{eq:system_matrices}). The block diagonal $H$ has been pre-empted in \cref{prop:vector-valued-time-varying-diffusion} above.

    \emph{(iii) Calibrate $\Gamma$:}
    The cost of assembling the quasi-MLE for $\Gamma_{n+1}$ according to \cref{eq:time-varying-scalar-diffusion} or \cref{eq:time-varying-vector-valued-diffusion} is $O(d)$, because the matrix to be inverted is diagonal.

    \emph{(iv) Extrapolate the covariance:}
    The covariance can be extrapolated dimension-by-dimension as well, because $C_n$, $\Phi(h_n)$, and $\Sigma(h_n)$ are all block diagonal with the same block structure: $d$ square blocks with $\nu + 1$ rows and columns; recall \cref{eq:extrapolate_covariance}.
    In reality, the matrix-matrix multiplication is replaced by a QR decomposition; we refer to \cref{sec:prediction_step_sqrt} for details on square root implementation.
    Using either strategy---square root  or traditional implementation---extrapolating the covariance costs $O(d \nu^3)$ and $C_{n+1}^-$ is block diagonal.

    \emph{(v) Measure:}
    Computing the mean of $Z_{n+1}$ (recall \cref{eq:measured_rv,eq:measured_rv_matrices}) costs $O(d)$.
    The covariance $S_{n+1}$ of $Z_{n+1}$ is diagonal, since $H$ and $C_{n+1}^-$ are  block diagonal.
    Thus, assembling \emph{and inverting} $S_{n+1}$ costs $O(d)$.

    \emph{(vi) Correct mean and covariance:}
    The mean is corrected according to \cref{eq:mean_correction}, which---since $S_{n+1}$ is diagonal---costs $O(d \nu)$.
    The covariance is corrected according to \cref{eq:joseph_update1,eq:joseph_update2}, the complexity of which hinges on the structure of $\Xi$ (\cref{eq:joseph_update2}): due to the block diagonal $C_{n+1}^-$, $H$, and $S_{n+1}$, $\Xi$ is block diagonal again, and correcting the covariance costs $O(d \nu^3)$.
    The square root matrix of $C_{n+1}$ arises by multiplying $\Xi$ with the ``left'' square root matrix of $C_{n+1}^-$.
    The complexity remains the same (asymptotically, though QR decompositions cost more than matrix multiplications).

    All in all, ODE filter steps preserve block-diagonal structure in the covariances.
    The expensive phases are the covariance extrapolation and correction in $O(d \nu^3)$ floating-point operations. The maximum memory demand is $O(d \nu^2)$ for the block diagonal covariances.
\end{proof}

While it may seem restrictive at first to use only the diagonal of the Jacobian, \cref{prop:complexity_diagonal_jacobian} includes the EK0, one of the central ODE filters.
The $O(d)$ complexity puts the EK0 and the diagonal EK1 into the complexity class of explicit Runge--Kutta methods. Usually, $\nu < 12$ holds \citep{kramer2020stable}.

\section{EK0 PRESERVES KRONECKER STRUCTURE}
\label{sec:kronecker_structure_preserved}

\begin{figure*}
    \begin{center}
        \includegraphics[width=\linewidth]{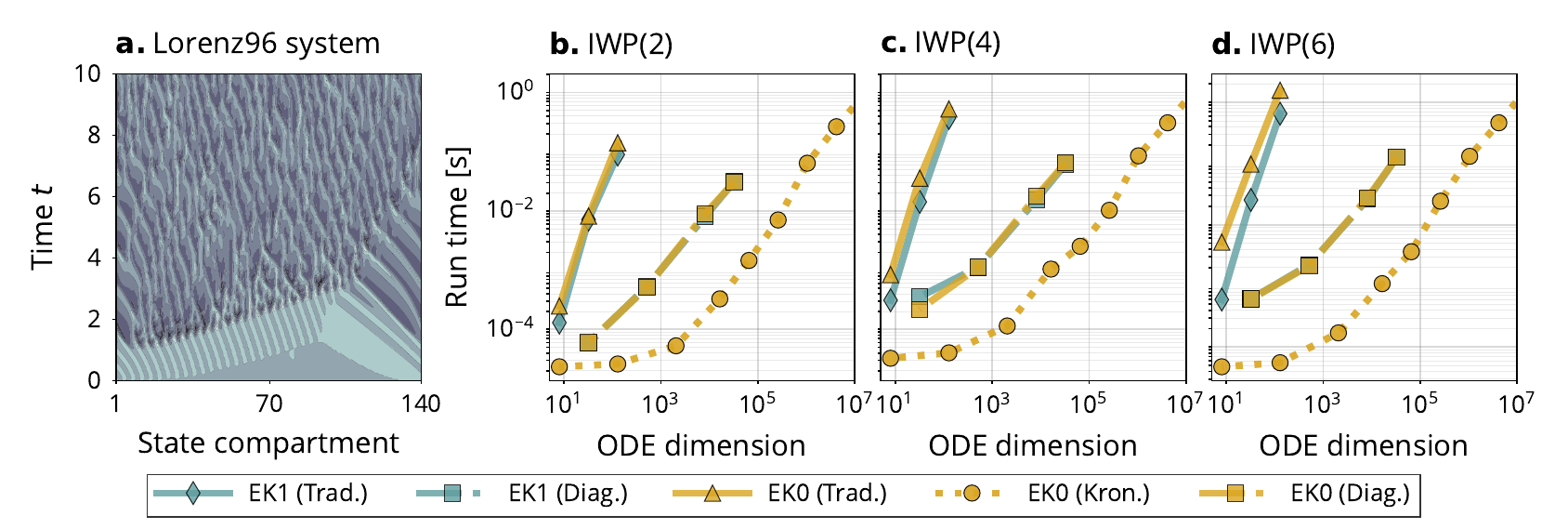}
        \caption{
            \textit{Runtime of a single ODE filter step:}
            Run time (wall-clock) of a single step of ODE filter variations on the Lorenz96 problem (a) for increasing ODE dimension and $\nu=2,4,6$ (b-d).
            The traditional implementations cost $O(d^3)$ per step, the diagonal EK1 and diagonal EK0 are $O(d)$ per step, just like the Kronecker version of the EK0.
            The Kronecker EK0 is significantly faster than the diagonal version(s).
        }
        \label{fig:1_sterilised_lorenz_attempt_step}
    \end{center}
\end{figure*}

As hinted in \cref{sec:introduction}, scalar or diagonal diffusion may be too restrictive in certain situations.

\begin{example}\label{ex:spatiotemporal-model}
    Consider a spatio-temporal Gaussian process model $u(t, x) \sim \gp(0, \gamma^2 k_t \otimes k_x)$, where $k_t$ is the covariance kernel that directly corresponds to an integrated Wiener process prior \citep{sarkka2019applied}.
    Such a spatiotemporal model could be a useful prior distribution for applying an ODE solver to problems that are discretised PDEs, because $k_x$ encodes spatial dependency structures. Restricted to a spatial grid $\Xcal := \{x_1, ..., x_G\}$, $y:= u(t, \Xcal)$ satisfies the prior model in \crefrange{eq:prior_sde}{eq:prior_initial_condition}\footnote{Technically, the stack of $y$ and its $\nu$ derivatives does.}, but with $\Gamma = \gamma^2 k_x(\Xcal, \Xcal)$ \citep{solin2016stochastic}, which is usually dense.
\end{example}

\subsection{Assumptions}
Despite the lack of independence in \cref{ex:spatiotemporal-model}, fast ODE solutions remain possible with the EK0.
In the remainder of this section, let $\Gamma = \gamma^2 \breve{\Gamma}$ for some matrix $\breve{\Gamma}$ and some scalar $\gamma$.
Calibrating the scalar $\gamma$ allows preserving Kronecker structure in the system matrices that appear in an ODE filter step (\cref{prop:complexity-kronecker} below).
\citet{tronarp2019probabilistic} show how for $\Gamma = \gamma^2 \breve{\Gamma}$, a time-constant quasi-MLE $\hat\gamma$ arises in closed form and also, that the posterior covariances all look like $C_n = \gamma^2 \breve{C}_n$: calibration can happen entirely post-hoc.

\paragraph{Constraints}
The following statement about linear complexity of ODE filters is only valid under two constraints: one can ignore (i) the quadratic costs of multiplying the posterior covariances with the quasi-MLE, and (ii) the cubic costs of solving a linear system involving $\Gamma$.
The system matrices $\Phi, \Sigma, C_0$ are all Kronecker products of a $\Rbb^{d \times d}$ (``left'') and a $\Rbb^{\nu \times \nu}$ factor (``right'').
The first constraint is thus avoided by scaling the ``right'' Kronecker factor of the covariances with $\gamma^2$  in $O(\nu^2)$.
(ODE filters preserve Kronecker structure; see below.)
The second one becomes the following assumption.

\begin{assumption}\label{ass:inverse-of-gamma-cheap}
    Assume that the inverse of $\Gamma$ is readily available and cheap to apply; that is, the quantity $x^\top \Gamma^{-1} x$ can be computed in $O(d)$.
\end{assumption}

Naturally, \cref{ass:inverse-of-gamma-cheap} holds for diagonal or at least sufficiently sparse matrices $\Gamma$.
There are also settings in which \cref{ass:inverse-of-gamma-cheap} holds even if $\Gamma$ is dense.
For instance, if $\Gamma$ is the covariance of a Gauss--Markov random field, the sparsity structure in $\Gamma^{-1}$ implies adjacency of grid nodes \citep{lindgren2011explicit,siden2020deep}. In \cref{ex:spatiotemporal-model} with a spatial Matern kernel, for example, inverse Gram matrices can be approximated efficiently using the stochastic partial differential equation formulation \citep{lindgren2011explicit}.

\subsection{Computational Complexity}

Under \cref{ass:inverse-of-gamma-cheap}, a single EK0 step costs $O(d)$:

\begin{proposition}\label{prop:complexity-kronecker}
    Under \cref{ass:inverse-of-gamma-cheap}, and if a time-constant diffusion model $\Gamma = \gamma^2 \breve{\Gamma}$ is calibrated via $\gamma$, a single step of the EK0 costs $O(\nu^3 + d \nu^2)$ floating point operations, and $O(d\nu + d^2 + \nu^2)$ memory.
\end{proposition}

The proof parallels that of \cref{prop:complexity_diagonal_jacobian} (details are in \cref{sec:proof-complexity-kronecker}).
It hinges on computing everything only in the ``right'' factor of each Kronecker matrix.
The proposition can be extended to time-varying diffusion if one tracks $\gamma$ in the ``right'' Kronecker factor instead of the ``left'' one. Since this obfuscates the notation, we prove the claim in \cref{sec:kronecker-time-varying-diffusion}.
The quadratic $O(d^2)$ memory requirement is entirely due to the cost of storing $\Gamma$---if $\Gamma$ or its inverse are banded matrices, for instance, it reduces to $O(d\nu + \nu^2)$.

\section{EMPIRICAL EVALUATION}
\label{sec:experiments}

\paragraph{A Single ODE Filter Step}
We begin by evaluating the cost of a single step of the ODE filter variations on the Lorenz96 problem \citep{lorenz1996predictability}.
This is a chaotic dynamical system and recommends itself for the first experiment, as its dimension can be increased freely.
\cref{sec:all-the-odes} contains more detailed descriptions of all ODE models.
We time a single ODE filter step for increasing ODE dimension $d$ and different solver orders $\nu \in \{2,4,6\}$.
The results are depicted in \cref{fig:1_sterilised_lorenz_attempt_step}.
The traditional EK0 and EK1 become infeasible due to their cubic complexity in the dimension.
The diagonal EK1 and the diagonal EK0 exhibit their $O(d)$ cost.
The Kronecker EK0 is cheaper than the independence-based solvers. A step with the Kronecker EK0 takes $\sim$1 second for a 16 million-dimensional ODE on a generic, consumer-level CPU.
Altogether,
\cref{fig:1_sterilised_lorenz_attempt_step} confirms \cref{prop:complexity_diagonal_jacobian,prop:complexity-kronecker}.

\paragraph{A Full Simulation}
Next, we evaluate whether the performance gains for a single ODE filter step translate into a reduced overall runtime (including step-size adaptation and calibration) on a medium-dimensional problem: the Pleiades problem \citep{hairer1993solving}.
It describes the motion of seven stars in a plane and is commonly solved as a system of 28 first-order ODEs.
The results are in \cref{fig:experiment_2_full_solves}.
\begin{figure}
    \begin{center}
        \includegraphics[width=0.47\textwidth]{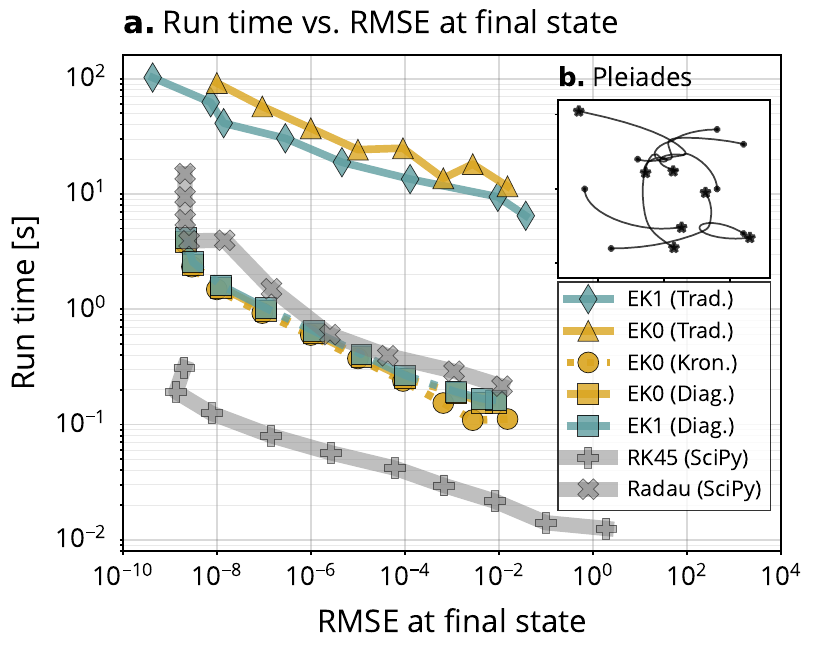}
    \end{center}
    \caption{\textit{Runtime efficiency of fast ODE filters:} Run time per root mean-square error of the ODE filters (a) on the Pleiades problem (b). The figure also shows two reference ODE filters, EK0 and EK1 in the traditional implementation, and Scipy's RK45 (explicit) and Radau (implicit). On the 28-dimensional Pleiades problem, the improved implementation accelerates the ODE filter implementations significantly.
    }
    \label{fig:experiment_2_full_solves}
\end{figure}
Pleiades reveals the increased efficiency of the ODE filters.
The probabilistic solvers are as fast as Radau, only by a factor $\sim$10 slower than SciPy's RK45 \citep{virtanen2020scipy}, but 100 times faster than their reference implementations. (It should be noted that the ODE filters use just-in-time compilation for some components, whereas SciPy does not.)

\paragraph{A High-Dimensional Setting}
\label{subsec:high-dimensional-settings}
To evaluate how well the improved efficiency translates to extremely high dimensions, we solve the discretised FitzHugh-Nagumo PDE model on high spatial resolution (which translates to high dimensional ODEs).
The results are in \cref{fig:experiment_3_pde_on_gpu}.
\begin{figure}[t!]
    \begin{center}
        \includegraphics[width=0.47\textwidth]{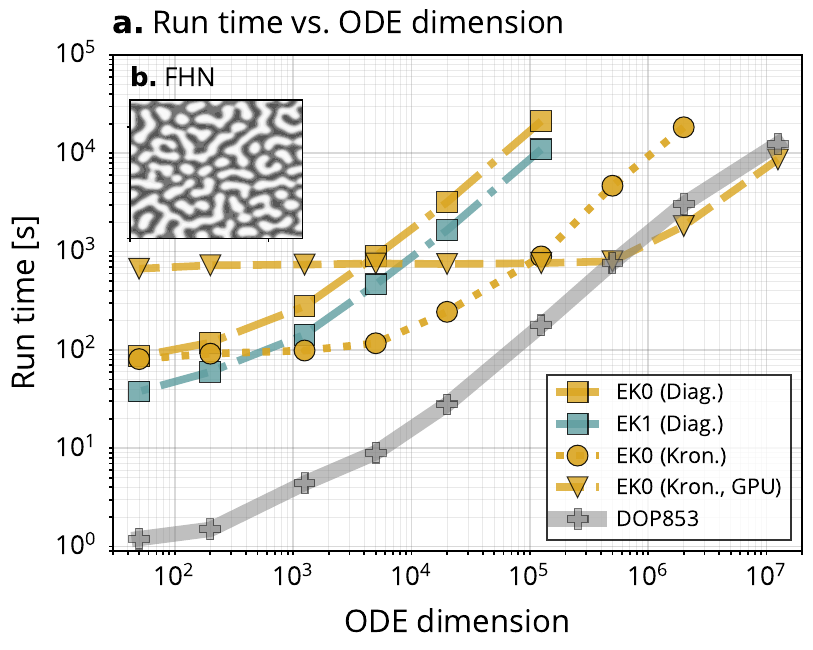}
    \end{center}
    \caption{\textit{High-dimensional PDE discretisation:}
        Run-time of ODE filters on the discretised FitzHugh-Nagumo model for increasing ODE dimension (i.e.\ increasing spatial resolution) including calibration and adaptive time-steps. SciPy's DOP853 for reference. Simulating  $\gg 10^6$-dimensional ODEs takes $\approx 3h$.
    }
    \label{fig:experiment_3_pde_on_gpu}
\end{figure}
The main takeaway is that ODEs with millions of dimensions can be solved \emph{probabilistically} within a realistic time frame (hours), which has not been possible before.
GPUs improve the runtime for extremely high-dimensional problems ($d \gg 10^5$).

\paragraph{Stability Of The Diagonal EK1}
\label{subsec:stability-diagonal-ek1}
How much do we lose by ignoring off-diagonal elements in the Jacobian?
To evaluate the loss (or preservation) of stability against the $A$-stable EK1 \citep{tronarp2019probabilistic}, we solve the Van der Pol system \citep{guckenheimer1980dynamics}.
It includes a free parameter $\mu > 0$, whose magnitude governs the stiffness of the problem: the larger $\mu$, the stiffer the problem, and for \eg $\mu =10^6$, Van der Pol is a famously stiff equation.
The results are in \cref{fig:experiment_4_stability_ek1}.
\begin{figure}[t!]
    \begin{center}
        \includegraphics[width=0.47\textwidth]{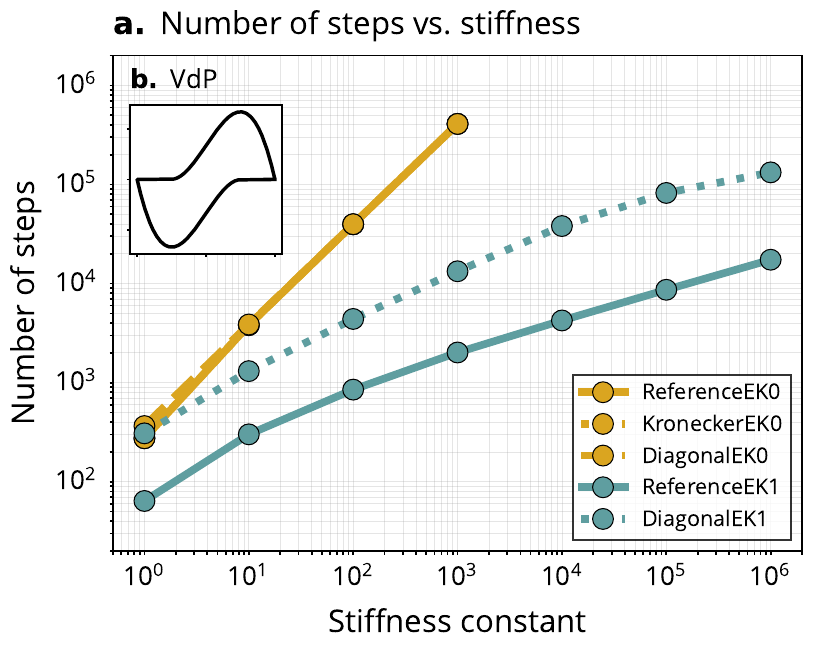}
    \end{center}
    \caption{\textit{Stability:} Number of steps taken by an ODE filter (a) for an increasingly stiff Van der Pol system (b). The diagonal EK1 is more stable than the EK0, but less stable than the EK1 (which is expected because it uses strictly less Jacobian information).}
    \label{fig:experiment_4_stability_ek1}
\end{figure}
We observe how the diagonal EK1 is less stable than the reference EK1 for increasing stiffness constant, but also that it is significantly more stable than the EK0, for instance.
It is a success that the diagonal EK1 solves the van der Pol equation for large $\mu$.

\section{CONCLUSION}

For probabilistic ODE solvers to capitalize on their theoretical advantages, their computational cost has to come close to that of their non-probabilistic point-estimate counterparts (which benefit from decades of optimization).
High-dimensional problems are one obstacle on this path, which we cleared here.
We showed that independence assumptions in the underlying state-space model, or preservation of Kronecker structures, can bring the computational complexity of a large subset of known ODE filters close to non-probabilistic, explicit Runge--Kutta methods.
As a result, probabilistic simulation of extremely large systems of ODEs is now possible, opening up opportunities to exploit the advantages of probabilistic ODE solvers on challenging real-world problems.

\subsubsection*{Acknowledgements}
The authors gratefully acknowledge financial support by the German Federal Ministry of Education and Research (BMBF) through Project ADIMEM (FKZ 01IS18052B).
They also gratefully acknowledge financial support by the European Research Council through ERC StG Action 757275 / PANAMA; the DFG Cluster of Excellence ``Machine Learning - New Perspectives for Science'', EXC 2064/1, project number 390727645; the German Federal Ministry of Education and Research (BMBF) through the Tübingen AI Center (FKZ: 01IS18039A); and funds from the Ministry of Science, Research and Arts of the State of Baden-Württemberg.
Moreover, the authors thank the International Max Planck Research School for Intelligent Systems (IMPRS-IS) for supporting Nicholas Krämer and Nathanael Bosch.

The authors thank Katharina Ott for helpful feedback on the manuscript.

\medskip
\bibliography{bibfile}

\appendix

\onecolumn
\aistatstitle{Probabilistic ODE Solutions in Millions of Dimensions: \\
  Supplementary Materials}

\section{SQUARE-ROOT IMPLEMENTATION OF PROBABILISTIC ODE SOLVERS}
\label{sec:sqrt_implementation}

The following two sections detail the square-root implementation of the transitions underlying the probabilistic ODE solver.
The whole section is a synopsis of the explanations by \citet{kramer2020stable}. See also the monograph by \citet{grewal2014kalman} for additional details.

\subsection{Extrapolation}
\label{sec:prediction_step_sqrt}
The extrapolation step
\begin{align}
    C^-_{n+1} = \Phi(h_n) C_n \Phi(h_n)^\top + \Sigma(h_n)
\end{align}
does not lead to stability issues further down the line (i.e. in calibration/correction/smoothing steps) if carried out in square root form.
Square-root form means that instead of tracking and propagating covariance matrices $C$, only square root matrices $C = \sqrt{C} \sqrt{C}^\top$ are used for extrapolation and correction steps without forming the full covariance.

This is possible by means of QR decompositions.
The matrix square root  $\sqrt{C_{n+1}^-}$ arises from $\sqrt{C_n}$ through the QR decomposition of
\begin{align}
    Q
    \begin{pmatrix}
        R \\
        0
    \end{pmatrix}
    \leftarrow
    \begin{pmatrix}
        \sqrt{C_n}\Phi(h_n)^\top \\
        \sqrt{\Sigma(h_n)^\top}
    \end{pmatrix},
    \quad \sqrt{C_{n+1}^-} \leftarrow R^\top
\end{align}
because
\begin{align}
    R^\top R
    =
    \begin{pmatrix}
        R \\
        0
    \end{pmatrix}^\top
    Q^\top Q     \begin{pmatrix}
        R \\
        0
    \end{pmatrix}
    =
    \begin{pmatrix}
        \sqrt{C_n}\Phi(h_n)^\top \\
        \sqrt{\Sigma(h_n)^\top}
    \end{pmatrix}^\top
    \begin{pmatrix}
        \sqrt{C_n}\Phi(h_n)^\top \\
        \sqrt{\Sigma(h_n)^\top}
    \end{pmatrix}
    =
    \Phi(h_n) C_n \Phi(h_n)^\top + \Sigma(h_n).
\end{align}
QR decompositions of a rectangular matrix $M \in \Rbb^{m \times n}$, $m > n$ costs $O(m n^2)$, which implies that the covariance square root correction costs $O(d^3\nu^3)$.
The QR decomposition is unique up to orthogonal row-operations (e.g. multiplying with $\pm 1$).
Probabilistic ODE solvers require only \emph{any} square root matrix, so this equivalence relation can be safely ignored -- they all imply the same covariance.

\subsection{Correction}

The correction follows a similar pattern.
Recall the linearised observation model
\begin{align}
    \Ical(y) \approx \hat\Ical(y) =  H y + b
\end{align}
where $H$ contains the vector field information and the Jacobian information (potentially, depending on linearisation style).
There are two ways of performing square root corrections: the conventional way, and the way that is tailored to probabilistic ODE solvers, which builds on Joseph form corrections.

\paragraph{Conventional Way}
Let $\sqrt{C^-}$ be a matrix square root of the current extrapolated covariance (we drop the $n+1$ index for improved readability).
Let $0_n$ be the $n \times n$ zero matrix, and $0_{n\times m}$ the $n \times m$ zero matrix.
The heart of the square root correction is another QR decomposition of the matrix
\begin{align}
    Q
    \begin{pmatrix}
        R_{11}                & R_{12} \\
        0_{d(\nu+1) \times d} & R_{22}
    \end{pmatrix}
    \leftarrow
    \begin{pmatrix}
        \sqrt{C^-}^\top H^\top & \sqrt{C^-}^\top \\
        0_{d(\nu+1)\times d}   & 0_{d(\nu+1)}
    \end{pmatrix}
\end{align}
with $R_{11} \in \Rbb^{d \times d}$, $R_{12} \in \Rbb^{d \times d(\nu+1)}$, and $R_{22} \in \Rbb^{d(\nu+1) \times d(\nu + 1)}$.
The $R_{ij}$ matrices contain the relevant information about the covariance matrices involved in the correction:
\begin{itemize}
    \item $\sqrt{S} = R_{11}^\top$ is the matrix square root of the innovation covariance
    \item $\sqrt{C} = R_{22}^\top$ is the matrix square root of the posterior covariance
    \item $K = R_{12} (R_{11})^{-1}$ is the Kalman gain and can be used to correct the mean
\end{itemize}
This QR decomposition costs $O(d^3 \nu^3)$ again, but the matrix involved is larger than the stack of matrices in the extrapolation step (it has $d$ more columns), so for high dimensional problems, the increased overhead becomes significant.
However, if only \emph{any} square root matrix is desired, this step can be circumvented.

\paragraph{Joseph Way}
Again, let $\sqrt{C^-}$ be the square root matrix of the current extrapolated covariance which results from the extrapolation step in square root form.
Next, the full covariance is assembled (which goes against the usual grain of avoiding full covariance matrices, but in the present case does the job) as $C^- = \sqrt{C^-} \sqrt{C^-}^\top$.
Since in the probabilistic solver, $C^-$ is either a Kronecker matrix $I_d \otimes \breve{C}^-$ or a block diagonal matrix $\blockdiagonal((C^-)^1, ..., (C^-)^d)$, this is sufficiently cheap.
The innovation covariance itself then becomes
\begin{subequations}
    \begin{align}
        S
         & = H C^- H^\top                                                                                        \\
         & = (E_1 - F_x E_0) C^- (E_1 - F_x E_0)^\top                                                            \\
         & = E_1 C^- E_1^\top - F_x E_0 C^- E_1^\top - E_1 C^- E_0^\top F_x^\top + F_x E_0 C^- E_0^\top F_x^\top
    \end{align}
\end{subequations}
which can be computed rather efficiently because $E_i C^- E_j^\top$ only involves accessing elements, not matrix multiplication.
The only non-negligible cost here is multiplication with the Jacobian of the ODE vector field (which is often sparse in high-dimensional problems).
Then, the Kalman gain
\begin{align}
    K = C^- H^\top S^{-1}
\end{align}
can be computed from $S$ which implies that the covariance correction reduces to
\begin{align}
    \sqrt{C} = (I - K H) \sqrt{C^-}
\end{align}
which is the ``left half'' of the Joseph correction.
The resulting matrix is square, and a matrix square root of the posterior covariance, but not triangular thus no valid Cholesky factor.
If the sole purpose of the square root matrices is improved numerical stability, generic square root matrices suffice.

\section{INDEPENDENCE AND FAST ODE FILTERS FOR TIME-CONSTANT DIFFUSION}
A similar result to \cref{prop:vector-valued-time-varying-diffusion} can be formulated for vector-valued, time-constant diffusion models.
\label{sec:independence-time-constant-diffusion}
\setcounter{proposition}{0}
\renewcommand{\theproposition}{\Alph{section}\arabic{proposition}}
\begin{proposition}
    Under \cref{ass:diagonal_gamma} and for diagonal $F_y$, a quasi-maximum likelihood estimate (MLE) for a vector-valued, time-constant diffusion model
    \(\Gamma = \diag((\gamma^1)^2, ..., (\gamma^d)^2)\)
    is given by the estimator
    \begin{align}\label{eq:time-varying-vector-valued-diffusion-supp}
        (\hat\gamma^i)^2 := \frac{1}{N} \sum_{i=1}^N \frac{(z_n^i)^2}{[S_n]_{ii}},
        \qquad i = 1, \dots, d,
    \end{align}
    where \(S_n := H(t_n) \Sigma(h_n) H(t_n)^\top \) is the diagonal covariance matrix of the measurement \(Z_n\) (recall \cref{subsec:information_model}).
\end{proposition}

\begin{proof}
    The proof is structured as follows.
    First, we show that an initial covariance
    \begin{equation}
        C_0 = \blockdiagonal((\gamma^1)^2 \breve{C}_0, \dots, (\gamma^d)^2 \breve{C}_0)
    \end{equation}
    implies covariances
    \begin{subequations}
        \begin{align}
            C_n^- & = \blockdiagonal \left( (\gamma^1)^2 (C_n^1)^-, \dots, (\gamma^d)^2 (C_n^d)^- \right), \\
            C_n   & = \blockdiagonal \left( (\gamma^1)^2 C_n^1, \dots, (\gamma^d)^2 C_n^d \right),         \\
            S_n   & = \diag \left((\gamma^1)^2 s_n^1, \dots, (\gamma^d)^2 s_n^d \right).
        \end{align}
    \end{subequations}
    Then, for measurement covariances \(S_n\) of such form, we can compute the (quasi) maximum likelihood estimate
    \(\hat{\Gamma}\).
    Because every covariance depends multiplicatively on \(\gamma\), calibration can happen entirely post-hoc.

    \paragraph{Block-Wise Scalar Diffusion}
    Recall from \cref{subsec:prior} that the transition matrix and the process noise covariance are of the form
    \(\Phi(h_n) = I_d \otimes \breve{\Phi}(h_n)\) and
    \(\Sigma(h_n) = \Gamma \otimes \breve{\Sigma}(h_n)\).
    Thus, for a diagonal diffusion
    \(\Gamma = \diag((\gamma^1)^2, ..., (\gamma^d)^2)\),
    both \(\Phi(h_n)\) and \(\Sigma(h_n)\) are block diagonal.
    Assuming a block diagonal covariance matrix that depends multiplicatively on \(\gamma\),
    \begin{equation}
        C_{n-1} = \blockdiagonal \left( (\gamma^1)^2 C_{n-1}^1, \dots, (\gamma^d)^2 C_{n-1}^d \right),
    \end{equation}
    the extrapolated covariance is also of the form
    \begin{subequations}
        \begin{align}
            C_n^-     & = \blockdiagonal \left((\gamma^1)^2 (C_n^1)^-, \dots, (\gamma^d)^2 (C_n^d)^- \right), \\
            (C_n^i)^- & := \breve{\Phi}(h_n) C_{n-1}^i \breve{\Phi}(h_n)^\top + \breve{\Sigma}(h_n),
            \qquad i = 1, \dots, d.
        \end{align}
    \end{subequations}
    The diagonal Jacobian \(F_y\) implies a block diagonal linearisation matrix
    \begin{subequations}
        \begin{align}
            H_n   & = E_1 - F_y E_0 = \blockdiagonal \left(H_n^1, \dots, H_n^d\right), \\
            H_n^i & := e_1 - [F_y]_{i,i} e_0,
            \qquad i = 1, \dots, d.
        \end{align}
    \end{subequations}
    The measurement covariance \(S_n\) is therefore given by a \emph{diagonal} matrix
    and depends multiplicatively on \(\gamma\), as
    \begin{subequations}
        \begin{align}
            S_n   & = H_n C_n^- H_n^\top
            = \diag \left((\gamma^1)^2 s_n^1, \dots, (\gamma^d)^2 s_n^d \right), \\
            s_n^i & := H_n^i (C_n^i)^- (H_n^i)^\top,
            \qquad i = 1, \dots, d.
        \end{align}
    \end{subequations}
    This implies a block diagonal Kalman gain
    \begin{subequations}
        \begin{align}
            \Xi_n   & = I - C_n^- H(t_n)^\top S_n^{-1} H(t_n) = \blockdiagonal \left(\Xi_n^1, \dots \Xi_n^d \right), \\
            \Xi_n^i & := I_{\nu+1} - (C_n^-)^i (H_n^i)^\top H_n^i / s_n^i
            \qquad i = 1, \dots, d.
        \end{align}
    \end{subequations}
    Finally, we obtain the corrected covariance
    \begin{subequations}
        \begin{align}
            C_n   & = \blockdiagonal \left((\gamma^1)^2 C_n^1, \dots, (\gamma^d)^2 C_n^d \right), \\
            C_n^i & := \Xi_n^i (C_n^i)^- (\Xi_n^i)^\top
            \qquad i = 1, \dots, d.
        \end{align}
    \end{subequations}
    This concludes the first part of the proof.

    \paragraph{Computing The Quasi-MLE}
    It is left to compute the (quasi) MLE
    \(\hat{\Gamma} = \diag((\hat{\gamma}^1)^2, ..., (\hat{\gamma}^d)^2)\)
    by maximizing the log-likelihood
    \({\log p(z_{1:N}) = \log \prod_{n=1}^N \mathcal{N} \left( 0; z_n, S_n \right)}\).
    Since
    \(S_n = \diag \left((\gamma^1)^2 s_n^1, \dots, (\gamma^d)^2 s_n^d \right)\)
    is a diagonal matrix, we obtain
    \begin{subequations}
        \begin{align}
            \hat{\Gamma}
             & = \argmax_\Gamma \sum_{n=1}^N \log \mathcal{N} \left( 0; z_n, S_n \right)                                                                        \\
             & = \argmax_\Gamma \sum_{i=1}^d \left( - \frac{N \log (\hat{\gamma}^i)^2 }{2} - \sum_{n=1}^N \frac{(z_n)_d^2}{2 s_n^i (\hat{\gamma}^i)^2} \right).
        \end{align}
    \end{subequations}
    By taking the derivative and setting it to zero, we obtain the quasi-MLE from
    \cref{eq:time-varying-vector-valued-diffusion-supp}.

\end{proof}

\section{PROOF OF \cref{prop:complexity-kronecker}}
\label{sec:proof-complexity-kronecker}

\begin{proof}

    Let $Y_n \sim \Ncal(m_n, C_n)$. Assume $C_n = \Gamma \otimes \breve{C}_n$ which is no loss of generality, because such a Kronecker structure is preserved through the ODE filter step as shown below.

    \emph{(i) Extrapolate mean:} The mean extrapolation costs $O(d \nu^2)$ like in the proof of \cref{prop:complexity_diagonal_jacobian}.

    \emph{(ii) Evaluate the ODE:} Evaluation of $H$ and $b$ is essentially free---recall that we only consider the EK0 in this setting, which uses the projection $H(t_n) = E_1$. Matrix multiplication with $H$ consists of a projection, which costs $O(1)$.

    \emph{(iii) Calibrate:} Calibration of a time-constant $\gamma^2$ costs $O(d)$ under \cref{ass:inverse-of-gamma-cheap}.

    \emph{(iv) Extrapolate covariance:}
    In the time-constant diffusion model, $\Sigma(h_n)$ and $C_n$ are both Kronecker matrices and share the left Kronecker factor: $\Gamma$. Thus, the extrapolation of the covariance can be carried out ``in the right Kronecker factor'', which costs $O(\nu^3)$ in traditional as well as square root implementation.
    Denote the extrapolated covariance by $C_{n+1}^- := \Gamma \otimes \breve{C}_{n+1}^-$.

    \emph{(v) Measure:}
    Recall $H(t_n) = E_1 = I \otimes e_1$.
    The mean of the measured random variable $Z_n \sim \Ncal(z_n, S_n)$ comes at negligible cost. The covariance
    \begin{align}
        S_{n+1}
        = H(t_{n+1}) C_{n+1}^- H(t_{n+1})^\top
        =  \Gamma \otimes\left[ e_1 \breve{C}_{n+1}^- e_1^\top\right]
    \end{align}
    requires a single element in $\breve{C}_{n+1}^-$. The Kalman gain
    \begin{align}
        K
        := C_{n+1}^- H(t_{n+1})^\top S_{n+1}^\top
        = I \otimes \breve{K},
    \end{align}
    with
    $        \breve{K} := e_1 \breve{C}_{n+1}^- / \left[ e_1\breve{C}_{n+1}^- e_1^\top\right]
    $    involves dividing the first row of $\breve{C}_{n+1}^-$ by a scalar. Its cost is $O(\nu+1)$.

    \emph{(v) Correct mean and covariance:}
    The mean is corrected in $O(d \nu^2)$ as in the proof of \cref{prop:complexity_diagonal_jacobian}.
    Due to the Kronecker structure in $K$, the ``left'' Kronecker factor of $C_{n+1}$ must be $\Gamma$ again.
    Therefore, we need to correct only the ``right'' Kronecker factor in $O(\nu^3)$.

    All in all, under the assumption of cheap calibration, a single step with the EK0 costs $O(d \nu^2)$ and the expensive steps are (as before) the covariance extrapolation and the covariance correction. The total memory costs are the requirements of storing $\Gamma$, the mean(s) in $O(\nu d)$, and the ``right'' Kronecker factor(s) in $O(\nu^2)$.
\end{proof}

\section{KRONECKER STRUCTURE AND FAST ODE FILTERS FOR TIME-VARYING DIFFUSION}
\label{sec:kronecker-time-varying-diffusion}
In the following we extend the results of
\cref{prop:complexity-kronecker}
to time-varying diffusion models.
\setcounter{proposition}{0}
\renewcommand{\theproposition}{\Alph{section}\arabic{proposition}}
\begin{proposition}
    Under \cref{ass:inverse-of-gamma-cheap}, and if a time-varying diffusion model $\Gamma_n = \gamma_n^2 \breve{\Gamma}$ is calibrated via $\gamma_n$, a single step of the EK0 costs $O(\nu^3 + d \nu^2)$ floating point operations, and $O(d\nu + d^2 + \nu^2)$ memory.
\end{proposition}

\begin{proof}
    The proof of \cref{prop:complexity-kronecker} shown in \cref{sec:proof-complexity-kronecker} depends on the specific time-fixed diffusion model only in the calibration (iii) and the extrapolation of the covariance (iv).
    In the following, we discuss these two steps for a time-varying diffusion \(\Gamma_n = \gamma_n^2 \breve{\Gamma}\).
    We show that Kronecker structure is preserved and we obtain the same complexities as in \cref{prop:complexity-kronecker}.

    \emph{(iii) Calibrate:}
    Calibration of a time-varying $\gamma_{n+1}^2$ is done with
    \begin{subequations}
        \begin{align}
            \hat \gamma_{n+1}^2
            : & = \frac{1}{d} z_{n+1}^\top \left[
                H(t_{n+1}) \Sigma(h_{n+1}) H(t_{n+1})^\top
                \right]^{-1} z_{n+1}                                                                                                        \\
              & = \frac{1}{d} z_{n+1}^\top \left(
            \breve{\Gamma}_{n+1} \cdot \left[ \breve{\Sigma}_{n+1}^- \right]_{11}
            \right)^{-1} z_{n+1}                                                                                                            \\
              & = \frac{1}{d} z_{n+1}^\top \left(\breve{\Gamma}_{n+1}\right)^{-1} z_{n+1} \Big/ \left[ \breve{\Sigma}_{n+1}^- \right]_{11},
        \end{align}
    \end{subequations}
    where we used that $H(t_n) = I \otimes e_1$.
    With \cref{ass:inverse-of-gamma-cheap}, this computation costs \(O(d)\).

    \emph{(iv) Extrapolate covariance:}
    Assume a covariance of the form
    \(C_n = \left( \gamma_n^2 \breve{\Gamma} \right) \otimes \breve{C}_n\).
    Since scalars can be moved between the Kronecker factors, the covariance matrix can be written with the diffusion matrix
    \(\Gamma_{n+1} = \gamma_{n+1}^2 \breve{\Gamma}\),
    as
    \begin{equation}
        C_n = \left( \gamma_{n+1}^2 \breve{\Gamma} \right) \otimes \left(\frac{\gamma_{n}^2}{\gamma_{n+1}^2} \breve{C}_n \right).
    \end{equation}
    Then, since
    \(\Sigma_{n+1} = \left( \gamma_{n+1}^2 \breve{\Gamma} \right) \otimes \breve{\Sigma}_{n+1}\),
    the prediction step can be written as
    \begin{subequations}
        \begin{align}
            C^-_{n+1} & = \Phi(h_{n+1}) C_{n} \Phi(h_{n+1})^\top + \Sigma(h_{n+1})                  \\
                      & =
            \left( I_d \otimes \breve{\Phi}(h_{n+1}) \right)
            \left(\left( \gamma_{n+1}^2 \breve{\Gamma} \right) \otimes \left(\frac{\gamma_{n}^2}{\gamma_{n+1}^2} \breve{C}_n \right)\right)
            \left( I_d \otimes \breve{\Phi}(h_{n+1}) \right)^\top
            +
            \left(\left( \gamma_{n+1}^2 \breve{\Gamma} \right) \otimes \breve{\Sigma}_{n+1} \right) \\
                      & =
            \left(\gamma_{n+1}^2 \breve{\Gamma} \right)
            \otimes \left(
            \frac{\gamma_{n}^2}{\gamma_{n+1}^2}
            \breve{\Phi}(h_{n+1}) \breve{C}_n \breve{\Phi}(h_{n+1})^\top +
            \breve{\Sigma}_{n+1} \right)                                                            \\
                      & =:
            \left(\gamma_{n+1}^2 \breve{\Gamma} \right)
            \otimes \breve{C}_{n+1}^-.
        \end{align}
    \end{subequations}
    With a Kalman gain of the form \(K = I \otimes \breve{K}\), the corrected covariance can be written as
    \({C_{n+1} = ( \gamma_{n+1}^2 \breve{\Gamma} ) \otimes \breve{C}_{n+1}}\),
    thus confirming our assumption on the Kronecker structure of covariance matrices.
    Since all matrix multiplications happen only ``in the right Kronecker factor'', extrapolating the covariance costs \(O(\nu^3)\).

    All other parts of the proof can be reproduced as in \cref{sec:proof-complexity-kronecker} to obtain the specified complexities.
\end{proof}

\section{ODE PROBLEMS}
\label{sec:all-the-odes}

\subsection{Lorenz96}
The Lorenz96 model describes a chaotic dynamical system for which the dimension can be chosen freely \citep{lorenz1996predictability}.
It is given by a system of \(N \geq 4\) ODEs
\begin{subequations}
    \begin{align}
        \dot{y}_1 & = (y_{2} - y_{N-1}) y_{N} - y_1 + F,                              \\
        \dot{y}_2 & = (y_{3} - y_{N}) y_{1} - y_2 + F,                                \\
        \dot{y}_i & = (y_{i+1} - y_{i-2}) y_{i-1} - y_i + F \qquad i = 3, \dots, N-1, \\
        \dot{y}_N & = (y_{1} - y_{N-2}) y_{N-1} - y_N + F,
    \end{align}
\end{subequations}
with forcing term \(F=8\), initial values \(y_1(0) = F + 0.01\) and \(y_{>1}(0) = F\), and time span \(t \in [0, 30]\).

\subsection{Pleiades}
The Pleiades system describes the motion of seven stars in a plane,
with coordinates \((x_i, y_i)\) and masses \(m_i = i\), \(i=1, \dots, 7\)
\citep[Section II.10]{hairer1993solving}.
It can be described with a system of \(28\) ODEs
\begin{subequations}
    \begin{align}
        \dot{x}_i & = v_i                                       \\
        \dot{y}_i & = w_i                                       \\
        \dot{v}_i & = \sum_{j \neq i} m_j (x_j - x_i) / r_{ij}, \\
        \dot{w}_i & = \sum_{j \neq i} m_j (y_j - y_i) / r_{ij},
    \end{align}
\end{subequations}
where \(r_{ij} = \left( (x_i - x_j)^2 + (y_i - y_j)^2 \right)^{3/2}\), for \(i,j=1,\dots,7\).
It is commonly solved on the time span \(t \in [0, 3]\) and with initial locations
\begin{subequations}
    \begin{align}
        x(0) & = [3,3,-1,-3,2,-2,2],    \\
        y(0) & = [3,-3,2,0,0,-4,4],     \\
        v(0) & = [0,0,0,0,0,1.75,-1.5], \\
        w(0) & = [0,0,0,-1.25,1,0,0].
    \end{align}
\end{subequations}

\subsection{FitzHugh--Nagumo PDE}
Let $\Delta = \sum_{i=1}^d \pderiv{^2}{x_i^2}$ be the Laplacian.
The FitzHugh--Nagumo partial differential equation (PDE) is  \citep{ambrosio2009propagation}
\begin{subequations}
    \begin{align}
        \pderiv{}{t} u(t,x) & = a \Delta u(t,x) + u(t,x) - u(t,x)^3 - v(t,x) + k, \\
        \pderiv{}{t}v(t,x)  & = \frac{1}{\tau}(b \Delta v(t,x) + u(t,x) - v(t,x))
    \end{align}
\end{subequations}
for $x \in [0,1] \times [0,1] \subseteq \Rbb^2$, some parameters $a,b,k,\tau$, and initial values $u(t_0, x) = h_0(x)$, $v(t_0, x) = h_1(x)$.
In our experiments, we chose $a = 208 \cdot 10^{-4}, ~ b=5 \cdot 10^{-3}, ~ k=-5 \cdot 10^{-3}, ~ \tau=0.1.$
As initial values, we used random samples from the uniform distribution on $(0,1)$.
We solve it from $t_0=0$ to $t_\text{max}=20$.
To turn the PDE into a system of ODEs, we discretised the Laplacian with central, second-order finite differences schemes on a uniform grid. The mesh size of the grid determines the number of grid points, which controls the dimension of the ODE problem.

\subsection{Van der Pol}
The Van der Pol system is often employed to evaluate the stability of stiff ODE solvers \citep{wanner1996solving}.
It is given by a system of ODEs
\begin{equation}
    \dot{y_1}(t) = y_2(t), \qquad
    \dot{y_2}(t) = \mu \left( \left( 1-y_1^2(t) \right) y_2(t) - y_1(t) \right),%
\end{equation}
with stiffness constant \(\mu > 0\), time span \(t \in [0, 6.3]\), and initial value \(y(0) = [2, 0]\).

\end{document}